\documentclass{article}

\usepackage[utf8]{inputenc} 
\usepackage{hyperref}       
\usepackage{url}            
\usepackage{booktabs}       
\usepackage{amsfonts}       
\usepackage{nicefrac}       
\usepackage{microtype}      
\usepackage{amsmath,amssymb}
\usepackage{amsthm}
\usepackage{graphics,graphicx}
\usepackage{subfigure}
\usepackage{enumerate}
\usepackage{algorithm,algorithmic}
\usepackage{tikz}
\usepackage{hyperref}
\usepackage{pspicture}
\usepackage{subfigure}

\newtheorem{theorem}{Theorem}
\newtheorem{definition}[theorem]{Definition}
\newtheorem{assumption}[theorem]{Assumption}

\newtheorem{example}[theorem]{Example}

\newcommand{\m}{\Theta} 
\newcommand{\calx}{\mathcal{X}}
\newcommand{\calu}{\mathcal{U}}
\newcommand{\calv}{\mathcal{V}}
\newcommand{\N}{\mathcal{N}} 
\newcommand{\tol}{\omega}

\newcommand{\R}{\mathbb{R}}

\newcommand{\1}{{\mathbf 1}}

\newcommand{\bfA}{\mathbf{A}}

\newcommand{\bfB}{\mathbf{B}}
\newcommand{\bfb}{\mathbf{b}}
\newcommand{\bfc}{\mathbf{c}}
\newcommand{\bfE}{\mathbf{E}}

\newcommand{\bfv}{\mathbf{v}}
\newcommand{\bfw}{\mathbf{w}}

\newcommand{\bfx}{\mathbf{x}}
\newcommand{\bfX}{\mathbf{X}}

\newcommand{\bfz}{\mathbf{z}}
\newcommand{\bfR}{\mathbf{R}}
\newcommand{\bfu}{\mathbf{u}}
\newcommand{\bfV}{\mathbf{V}}

\newcommand{\bfQ}{\mathbf{Q}}
\newcommand{\bfM}{\mathbf{M}}

\newcommand{\pv}{\mathcal{S}}
\newcommand{\spt}{\mathbf{s}}

\newcommand{\maxsafeball}{\mathsf{MaxSafeBall}}
\newcommand{\safeball}{\mathsf{SafeBall}}
\newcommand{\svd}{\mathsf{SVD}}
\newcommand{\sact}{\mathbf{u}^*}

\DeclareMathOperator*{\Exp}{\mathbb E}
\DeclareMathOperator*{\Var}{\mathsf{Var}}

\newcommand{\bfI}{\mathbf{I}}

\newcommand{\veps}{\varepsilon}

\newcommand{\bfxi}{\boldsymbol{\xi}}

\newcommand{\commentout}[1]{}

\title{Safe Exploration for Identifying Linear Systems via Robust Optimization}

\author{Tyler Lu \and Martin Zinkevich \and Craig Boutilier \and Binz Roy \and Dale Schuurmans\\Google}

\begin{document}

\date{}
\maketitle

\begin{abstract}
Safely exploring an unknown dynamical system is critical to the
deployment of reinforcement learning (RL) in physical systems where
failures may have catastrophic consequences. In scenarios where one knows
little about the dynamics, diverse transition data covering
relevant regions of state-action space is needed to apply either 
model-based or model-free RL. Motivated by the cooling
of Google's data centers,
we study how one can safely identify the parameters of a
system model with a desired accuracy and confidence level. In
particular, we focus on learning an unknown linear system with
Gaussian noise assuming only that, initially,
a \emph{nominal safe action} is known.
Define safety as satisfying specific linear constraints on the state space
(e.g., requirements on process variable) that must hold over the span of
    an entire trajectory, and given a Probably Approximately
    Correct (PAC) style bound on the estimation error
of model parameters, we show how to compute safe regions of action
space by gradually growing a ball around the nominal safe action. One
can apply \emph{any exploration strategy} where actions are chosen from
such safe regions. Experiments on a stylized model of data center
cooling dynamics show how computing proper safe regions can
increase the sample efficiency of safe exploration.
\end{abstract}

\section{Introduction}
\label{sec:intro}

Applying reinforcement learning (RL) to control physical systems in
complex, failure-prone environments
is non-trivial. The advantages of RL is significant: deploying an
agent that is able to learn and optimally control a sophisticated
physical system reduces the amount of manual tuning, hard-coding
and guess work that goes in to designing good controllers in many
applications.

However, a key problem in an unknown system is that the RL agent must
explore to learn the relevant aspects of the dynamics for optimal
control. Furthermore, in a failure-sensitive environment, the agent
must explore without triggering unsafe events. This is known as the
\emph{safe exploration} problem and has recently gained prominence in
both the research community and the broader public
consciousness~\cite{garcia:jmlr2015,amodei:2016,pecka:mesas14}.  This
interest has also been driven by some recent successes of RL over
traditional control regimes for physical systems such as those for
controlling pumps, chillers and cooling towers for data
centers~\cite{gao:2014}.

In this paper, we focus a critical aspect of model-based RL
control, namely, \emph{system identification subject to
safety constraints}. Specifically, working with linear Gaussian 
dynamics models---motivated by a data center cooling system control
at Google\footnote{We discuss this motivation in detail below.}---we 
define \emph{safety} w.r.t.\ a domain-specific
set of linear constraints over state variables that must be
satisfied at all times. This is the common notion of safety in
industrial applications where certain \emph{process variables} (system state
variables being monitored or controlled) should
meet certain target set points, and must not diverge more than a fixed
amount from those points. Once a system dynamics model is learned
to a sufficient degree, it can be used
in a model-based RL or other control algorithm. 

Unlike recent work that takes a Bayesian approach to modelling system
uncertainty, we adopt a strict uncertainty perspective, and
reason with a set of models within which the true system
model must lie with high probability.  The set itself is defined using
PAC-style bounds derived from an exploration data set over
state-action space. With this perspective,
we compute safe regions of action space where, starting with an 
initial state, one can take any sequence of actions within the
safe region and guarantee satisfaction of the (state)
safety constraints on the states at every step of the trajectory.

Our approach assumes the existence of a \emph{nominal safe action}, namely,
an action that is safe in the sense that it, if continuously applied,
will eventually drive the system to a safe state no matter what the initial
state, and will maintain safety in steady state.
A safe region is then computed by gradually
growing a ball around the nominal action, or some other action known
to be safe. Our experiments use a modified version of the
linear-Gaussian dynamics of a data center's cooling system. We 
show how our algorithms can be used to
efficiently and safely explore while also computing high probability
bounds on the estimated model quality.

\section{Related Work}

Given our definition of safety, perhaps the most relevant work is that
of Akametalu et al.~\cite{tomlin:cdc2014}, who require that, during
exploration, the states at every time step must reside within a
pre-defined region of safety.  They focus on computing a subset of the
safe states $D$ with the property that, when starting from any initial
state in $D$, one can execute any exploration policy as long as
certain safe actions are taken when the state starts drifting out of
$D$. However, computation of $D$ scales exponentially in the state
dimension. The transition dynamics are assumed to be partially known
but a deterministic disturbance function is unknown and modeled as a
Gaussian process (GP). While their safety definition bears a close
resemblance to ours, we determine safe regions of action space as a
function of the initial state rather than computing safe initial
states from which to explore. Furthermore, we do not assume any prior
knowledge of the dynamics model and instead treat model uncertainty
with PAC-like bounds on the quality of least-squares parameter
estimates.  Safe exploration with a finite state space, known
deterministic transition and an uncertain safety function modeled with
a GP is studied by Turchetta et al.~\cite{tbk:nips2016}, who develop
an exploration algorithm with provable safety guarantees. Several
other lines of work also use GPs for safe
exploration~\cite{btsk:arxiv2017,sui:icml2015,bs:ecc2015}.  Our
approach is not Bayesian, and does not explicitly encounter
computational complexity issues often associated with Bayesian
inference.

Our approach, using robust optimization of a worst-case safety violation,
is somewhat related to robust MDPs (RMDPs)
\cite{nilim:or05,iyengar-robustDP:mathOR05}.  In RMDPs one assumes the
true model lies in a smaller model set $\m$ defined by an error upper
bound on the transition probabilities, then the goal is to optimize a
policy that performs best with respect to an adversarial choice of the
true transition model in $\m$. We also work with a model set $\m$
where the true model lies with high probability but instead of
optimizing a policy (where one can also encode safety constraints) we
are interested in finding safe action regions. Recent work
\cite{gpc:nips16} considers the RMDP setting where the safety goal
is to find a policy that performs better than some baseline
policy, with most of the analysis in the finite-state setting, and methods
that do not easily scale to continuous, multi-dimensional
state spaces. Similar work~\cite{garcia:jair2012,kw:ei2011} also tries to
improve on a baseline policy. We do not assume a safe baseline policy
exists, but instead assume a known safe action (see above).

System identification is a well-studied topic in control
theory~\cite{pintelon:2012,ljung:1999,astrom:1971}. Typically it
consists of  three phases: (1) data gathering; (2) model class selection;
and (3) parameter estimation~\cite{ljung:1999}. In this work,
we primarily deal with in safe exploration for data gathering 
and parameter estimation in the context, quantifying the
approximation quality of parameter estimates and using this as a
stopping criterion for exploration. With regard to model class selection, 
various model types are studied in the literature, including
ARX (autoregressive with exogenous inputs), ARMA
(autoregressive moving average) and ARMAX (autoregressive moving
average with exogeneous inputs)~\cite{box:2015}. Our paper focuses on
learning linear ARX models as opposed to non-linear models, such as
neural networks~\cite{chen:ijc90}. For exploration, the celebrated
result on \emph{persistently exciting} inputs~\cite{astrom:1966}
asserts that control input signals that are rich enough---having a
certain number of distinct frequencies---provide sufficiently diverse
training data to learn the true model.  We leverage this insight in
our work.

In the context of robust parameter estimation, there is a line of work
on bounded noise models in system
identification~\cite{witsenhausen:tac68,schweppe:tac68,wpl:mcs90}. This
paradigm does not assume any statistical form of the noise inherent in
state observations, but instead assumes only that the noise is bounded
by an unknown constant. Researchers have studied the existence and
characteristics of estimators for various optimality
concepts~\cite{milanese:automatica91}.  In our setting, we assume a
statistical distribution for the noise model, but we compute robust,
safe actions with respect to the uncertainty in the model parameters
estimated from trajectory data.

\section{Preliminaries}

We first describe some basic concepts and notation related to
sequential decision making. Since our paper's main results center on
linear transition models with Gaussian noise, we address
this model class in some detail and outline some corresponding PAC-style bounds
for estimating the parameters of these models. These bounds depend on
properties of the training data---generally, a time series of system
behaviour---and thus do not require that training examples be
independently generated.

\subsection{Background and Notation}


%

Let $\calx \subseteq \R^d$ be the \emph{state space} and
$\calu\subseteq \R^{d'}$ the \emph{action space}. We generally view
each dimension of the state as representing some process variable to
be monitored or controlled, and each action dimension as some control
variable to be set.
For example, the
state might be a vector of sensor readings while the action might be a
vector of hardware parameter configurations.\footnote{Sensor
readings are themselves usually only stochastically related to 
an underlying (unobserved) state, hence in general we may need to maintain
state over some (Markovian) sufficient statistic summarizing the history
of past sensor readings. We occasionally equate sensor readings and state
only for illustration.}
Let $\bfx_t \in \calx$
denote the system state at time $t$ and $x_{t\ell}$ the value of its
$\ell$-th variable, for $t\in\{0,1,2,\ldots\}$. Likewise, let
$\bfu_t\in\calu$ be the action taken at time $t$ in state
$\bfx_t$. There is an underlying, unknown \emph{transition model}
$P^*(\bfx_{t+1} ~|~ \bfx_{t}, \bfu_t)$, i.e., the distribution from which
the state at the next time step $\bfx_{t+1}$ is generated. The
transition model is stationary and Markovian (i.e., depends only on the
current state $\bfx_t$ and action $\bfu_t$).

We consider time-based \emph{episodes} or \emph{trajectories}
of state-action sequences:
$(\bfx_0, \bfu_0), \ldots, (\bfx_{T-1}, \bfu_{T-1}), (\bfx_T)$.
A length $T$ episodes is typically generated by providing
\emph{initial state} $\bfx_0$ to our agent, which
executes an exploration policy that selects/executes action
$\bfu_t$ in response to observing state $\bfx_t$ (and past history), which
induces $\bfx_{t+1}$ according the true system dynamics, for
$T-1$ stages.
The initial state $\bfx_0$
might be exogeneously realized and may differ across episodes. Our
goal is to ensure safety during an episode.


We use the following notation throughout.
Let $\bfv_{i:j} = \bfv_i,\ldots,\bfv_j$ denote a sequence of $j-i+1$
vectors starting at index $i$. Let $\bfM$ be a matrix. We use $\bfM'$
to denote the matrix transpose. We use $\|\bfM\|_F$ to denote the
Frobenius norm, and for a vector $\bfv$, we use $\|\bfv\|$ to denote
the 2-norm.  We denote by $\bfM[\ell,:]$ the $\ell$-th row vector of
$\bfM$. For two row vectors $\bfv$ and $\bfw$, let
$\begin{bmatrix}\bfv & \bfw\end{bmatrix}$ be the row vector obtained
  by concatenating $\bfv$ and $\bfw$.

\subsection{Linear Gaussian Transition Models and Estimation Bounds}

Our goal is to identify the underlying model $P^*$. We
restrict this model to belong to a class of transition models
$\{P_{\theta} ~|~ \theta\in\m \}$ defined by a parameter space
$\m$. In particular, we assume there is some $\theta^*\in\m$ such
that $P^* = P_{\theta^*}$.  
We assume $\m$ corresponds to parameters of a
\emph{linear-Gaussian model}, though our notion of safety and robust
optimization formulations apply more generally. Specifically, the next
state is given by
\begin{eqnarray*}
  \bfx_{t+1} = \bfA \bfx_t + \bfB \bfu_t + \bfxi_t \quad \forall t\ge 0
\end{eqnarray*}
where $\bfA\in \R^{d\times d}$ and $\bfB\in\R^{d\times d'}$. 
A constant additive bias vector can be incorporated in $\bfB \bfu_t$ by
adding an extra dimension to $\bfu_t$ with a constant value of $1$. 
We assume a Gaussian random noise vector $\bfxi_t \sim \N({\bf 0},
\sigma^2 \bfI_d)$ where $\bfI_d$ is the $d\times d$ identity
matrix. We also assume $\bfxi_t$ is independent of $\bfxi_{t'}$ for
any $t'\neq t$ and is independent of the previous state $\bfx_{t-1}$.
For ease of exposition,
we omit $\sigma$ from $\m$---it can be easily
estimated as we describe below. 
Thus $\m = \{ \begin{bmatrix}\bfA &\bfB\end{bmatrix} ~|~
  \bfA \in\R^{d\times d}, \bfB\in \R^{d\times d'} \}$. We denote by
  $\theta[\ell,:] = [\bfA[\ell,:]~ \bfB[\ell,:]]$ the $\ell$-th row of
  the concatenated matrix.

Suppose we have existing time series $\left(\bfx_i,
\bfu_i\right)_{i=0}^{n-1}$ data generated by a (possibly randomized)
exploration strategy.  Let $\bfX_n$ be the ($n$-row) data matrix whose
$i$th row is vector $\bfx_{i-1}'$. Suppose, furthermore, that the
exploration strategy
meets the \emph{persistence of excitation} condition
\cite{ljung:1999}. That is, the matrix of auto-covariances of the
actions in the data is positive definite (e.g., independent, uniform
random actions for each $t\ge 0$ satisfy this property). Then the
least-squares estimates $\hat{\theta}=\begin{bmatrix}\hat{\bfA}&
\hat{\bfB}\end{bmatrix}$, where the predicted labels are the next
states, asymptotically converges to the following multivariate
Gaussian distribution (see, e.g., \cite{greene:2018}):
\begin{equation}
\hat{\theta}[\ell, :]\sim \N\left(\theta^*[\ell,:], \frac{\sigma^2 \bfQ^{-1}}{n}\right) \quad\forall \ell\le d,
\label{eq:normal_conv}
\end{equation}
where $\bfQ = \lim_{n\to\infty} \frac{1}{n}\bfX_n'\bfX_n$, which we assume
is positive definite.  Let $\Delta\theta = \theta - \theta^*$ and
$\Delta_\ell\theta = \theta[\ell,:]- \theta^*[\ell,:]$.  We can get an
elliptical confidence bound on our estimates from the Gaussian
distribution in~(\ref{eq:normal_conv}). Specifically, the
distribution over trajectories---induced by the 
randomness in 
linear Gaussian transitions (with respect to $\theta^*$) and any potential
possibly randomness in the exploration strategy---satisfies:
\[ \Pr
\left\{ \left(\bfx_i, \bfu_i\right)_{i=0}^{n-1} ~:~ \Delta_\ell\hat{\theta}'
\frac{n\bfQ}{\sigma^2} \Delta_\ell\hat{\theta} < F^{-1}(1-\alpha)
\right\} \ge 1-\alpha
\]
where $F^{-1}$ is the inverse CDF of the chi-squared distribution with
$d+d'$ degrees of freedom. Let $\lambda_{\min}(\bfQ)$ be the smallest
eigenvalue of $\bfQ$, which must be positive since $\bfQ$ is positive
definite. We can derive a looser confidence bound using the
Frobenius norm of $\Delta_\ell\hat{\theta}$:
\[ \Pr
\left\{\left(\bfx_i,\bfu_i\right)_{i=0}^{n-1} ~:~ \|\Delta_\ell\hat{\theta}\| < \sigma\sqrt{\frac{F^{-1}(1-\alpha)}{n\lambda_{\min}(\bfQ)}} \right\} \ge 1-\alpha.
\]
To make this bound practical, we can approximate $\bfQ$ and $\sigma$
with:
\[
\hat{\bfQ} = \frac{\bfX_n'\bfX_n}{n}, \qquad
\hat{\sigma} = \sqrt{\frac{\sum_{i=0}^{n-1} \|\bfx_{i+1} - (\hat{\bfA}\bfx_i +\hat{\bfB}\bfu_i) \|^2}{d(n-d-d')}}.
\]
Note that each row of $\hat{\theta}$ converges at a rate of
$O(1/\sqrt{n})$. If we apply a union bound over the $d$ rows of
$\hat{\theta}$, we can get an approximation bound that holds over 
the entire parameter matrix $\begin{bmatrix}\bfA &\bfB\end{bmatrix}$:
\begin{equation}
\Pr \left\{ \left(\bfx_i,\bfu_i\right)_{i=0}^{n-1} ~:~ \|\Delta\hat{\theta}\|_F <
\sigma \sqrt{\frac{dF^{-1}(1-\frac{\alpha}{d})}{n\lambda_{\min}(\bfQ)}}
\right\} \ge 1-\alpha.
\label{eq:est_bound}
\end{equation}


\section{A Safe Exploration Model}

\label{sec:robust}

We define \emph{safety} as satisfying relevant constraints on the system's
state. This is common, for example, in industrial applications where
certain process variables must be kept in a pre-specified
range. For instance, in data center cooling, the cold-air temperature
(i.e., prior to passing over computing hardware to be cooled) may not be
allowed to exceed some upper bound (safety threshold). This limit could
could be applied to all temperature sensors, the average, or the median or
some other quantile (where aggregations may be defined over various regions
of the DC floor).

Let $\pv \subseteq \{1,\ldots, d\}$ be the indices of process variables
to which safety constraints apply. We adopt a straightforward safety model
in which a single
inequality constraint is imposed
on each such process variable, all of which must be
satisfied:
\[
x_\ell \le s_\ell \quad \forall \ell\in\pv.
\]
Let $\spt \in \R^d$ be the vector where $\spt_\ell = s_\ell$ if
$\ell\in\pv$ otherwise $\spt_\ell = U$ for some upper bound $U$ on
$x_\ell$. We use this formulation for ease of exposition and
note that interval constraints on $x_\ell$ can be accommodated by
creating another state variable $-x_\ell$ with the desired bound.

At a high-level, our robust model of safety requires that, for
\emph{any model} $M$ within a specific region of model space where 
the true model resides with high probability, the
exploration actions we take for the purposes of system identifcation
result in satisfying safety conditions according to 
$M$'s dynamics.

\begin{definition}
Let $P^*$ be the true but unknown transition model. Let
$\spt\in \R^d$ be a vector of upper bound constraints on the state
variables. Let $\bfx_0$ be the initial state. We say that a \emph{
  sequence of actions $\bfu_{0:T-1}$ is safe in expectation} if
\begin{align}
  \Exp_{P^*}[\bfx_{\tau}] &\le \spt, \quad \forall \tau\in\{1,\ldots, T\}
\end{align}
where $\bfx_{t+1} \sim P^*(\cdot ~|~ \bfx_t, \bfu_t)$ for all
$t\geq 0$. It is \emph{safe with probability at least $1-\gamma$} if
\begin{align}
  \Pr_{\bfx_{1:T}}(\forall \tau \in \{1,\ldots, T\} \quad \bfx_\tau \le \spt) \geq 1-\gamma \label{eq:safewhp}
\end{align}
\end{definition}
The high probability guarantee is much stronger
than safety in expectation. However, in practice it may be
more difficult computationally to verify safe actions with high
probability; and safety in expectation might suffice if the true
model $P^*$ is close to the mean estimated model
(e.g., when transitions are nearly deterministic).

If we can obtain a high probability bound on each state at time
$\tau$, i.e. $\Pr(\bfx_\tau \le \spt) \geq 1 - \gamma/T$, then
by applying the union bound we can satisfy
inequality Eq.~(\ref{eq:safewhp}). We focus on this simpler safety 
requirement in the sequel. In practice, the union bound can be
quite conservative given state dependencies across time.

For practical exploration, we would like to take different sequences
of actions (e.g., persistently exciting) within a region of action
space that guarantees safety.  We assume there is a \emph{nominal
action} $\sact$ that is known to be safe with high probability when
applied consistently (i.e., at each time step).
\begin{assumption}
\label{ass:safe}
Let $\spt$ be the safety bounds, $T$ the time horizon, and $\gamma >
0$ the safety confidence parameter.  There exists a set of initial
states $\bfX_0 \subseteq \calx$, and a \emph{nominal safe action}
$\sact$ such that for all initial states $\bfx_0 \in \bfX_0$,
\[
\Pr \left( \forall \tau \le T \quad \bfx_\tau\le \spt \right) \ge 1-\gamma,
\]
where $\bfx_{t+1} \sim P^*(\cdot~|~\bfx_t, \sact)$ for all $t\ge 0$.
\end{assumption}
Such safe actions exist in many applications. For instance, in DC cooling
it is generally feasible (though expensive) to set cooling fans and water
flow to very high values (subject to their own constraints) to ensure cold air
temperatures are maintained at or below safety thresholds. It is also the
case that the nominal safe action will eventually drive
``routine'' unsafe states (i.e., those not associated with some significant
system failure) into the safe region.

The nominal safe action is also useful
for identifying the standard deviation parameter $\sigma$ of our
Gaussian noise---we can execute
the safe action for sufficiently long period (in practice,
the action can be safe in steady state).
For simplicity of analysis and exposition, we take
$\hat{\sigma} = \sigma$ to be known. Our results can be extended easily if we
remove this assumption by replacing $\sigma$ with an estimate
using a high-probability upper bound.

\begin{definition}
\label{def:safe-region}
We say that $\calv\subseteq \calu$ \emph{is a safe action region
  in expectation (respectively, with probability at least $1-\gamma$)}
if, for all $\bfu_{0:T-1}\in \calv^T$, $\bfu_{0:T-1}$ is safe in
expectation (respectively, safe with probability at least $1-\gamma$).
\end{definition}
A safe action region $\calv$ represents a region in which we can
explore freely, without having to re-compute whether a chosen sequence
of actions from this region is safe. In practice, we can continuously
select exploratory actions from the safe region.
Our safety definition guarantees with high probability that no state
in a trajectory of length $T$ is unsafe. However, if we re-run the
trajectory many times---i.e. run the same sequence of actions 
starting at the same initial state---the probability of an unsafe
event compounds. In practice, however, the additive noise term for the next state is
typically small
and bounded that states are unlikely to be unsafe. This is the
consequence of assuming a Gaussian noise model---if we draw $N$
samples from a Gaussian, for a sufficiently large $N$, we are likely
to draw a very rare event even if its probability is extremely small.


\section{One-step Safety}

While we have defined safety with respect to
constraint satisfaction across an entire trajectory,
we first analyze a simple case, namely,
whether an action is \emph{one-step safe}. Specifically, 
we assess, if we start
in a safe state $\bfx_0$, whether a given action $\bfu$ results in a
safe state $\bfx_1 \sim P_{\theta^*}(\cdot ~|~ \bfx_0, \bfu)$.

Suppose we have training data generated by past exploration, and
that $\|\Delta \hat{\theta}\|_F < \veps$ holds with probability
$1-\alpha$. This ensures that, with probability $1-\alpha$, should we
verify $\bfu$ is safe for all models
$\theta = \begin{bmatrix}\bfA & \bfB\end{bmatrix}$ within $\veps$
of the estimated model $\hat{\theta}$, that $\bfu$ is safe
w.r.t.\ the true model $\theta^*$. This
worst-case, robust analysis
ensures safety even for adversarially chosen $\theta$ within this $\veps$-ball.

We can formulate the an optimization problem to test such
one-step safety. The optimization in Eqs.~(\ref{eq:one1}-\ref{eq:one2}) 
finds the model parameters $\theta$ (i.e., system dynamics $[\bfA, \bfB]$)
within the high-probability $\veps$-ball around $\hat{\theta}$ such that
the safety constraint defined by any $\spt_\ell$ is maximally
violated (if violation is feasible).
Specifically,
for each process variable $\ell\in\pv$, we check whether the optimal
objective value is greater than $\spt_{\ell}$. If no such
value is greater,
then $\bfu$ is one-step safe in expectation.
\begin{align}
  \max_{\bfA, \bfB} & \quad \bfA[\ell,:] \bfx_0 + \bfB[\ell, :]\bfu 
  \label{eq:one1} \\
  \textrm{subject to} & \quad \left\| \begin{bmatrix}\bfA &\bfB\end{bmatrix} -
    \begin{bmatrix}\hat{\bfA} & \hat{\bfB}\end{bmatrix}\right\|_F \le \veps
  \label{eq:one2}
\end{align}

We can reformulate this problem to simplify the optimization.
Let $\bfc =
\frac{1}{\veps}([\bfA[\ell,:] ~\bfB[\ell,:]]' -
     [\hat{\bfA}[\ell,:]~\hat{\bfB}[\ell,:]]')$. 
To check safety in expectation, we can 
re-formulate Eqs.~(\ref{eq:one1}-\ref{eq:one2}) as:
\begin{align*}
  \max_{\bfc : \|\bfc\| \le 1} & \quad
  \left\|
    [\bfx_0' ~ \bfu']
    \bfc \right\| 
\end{align*}
This solution to this problem is the largest singular value of
$\begin{bmatrix}\bfx_0' & \bfu'\end{bmatrix}$, which is its vector $2$-norm, 
and the maximizing
$\bfc^*$ is the corresponding right-singular vector. Thus, the optimal
objective of the original optimization is 
$$\veps \left(\|[\bfx_0'
  ~\bfu']\| + \hat{\bfA}[\ell,:] \bfx_0 + \hat{\bfB}[\ell,:] \bfu
\right).$$ 
The maximizing $\bfA[\ell,:],\bfB[\ell,:]$ of the original
problem can also be easily derived. 


Notice that one-step safety does not imply subsequent states along
the trajectory will be safe, even if the nominal safe action is
taken. In fact, a one-step safe action may lead to a state where there are 
no safe actions.
\begin{example}[One-step safety insufficient]
Define a system whose true dynamics is given by:
\[
\bfA^* =
\begin{bmatrix}
  1/2 & 1\\
  0 & 0
\end{bmatrix},
\quad
\bfB^* = 
\begin{bmatrix}
  0&0\\
  1&0
\end{bmatrix},
\quad
\bfx_0 =
  \begin{bmatrix}
    1/2 & 0
  \end{bmatrix},
\quad
\bfu_0 =
\begin{bmatrix}
  1&0
\end{bmatrix}.
\] 
Suppose our controller/agent knows $\bfA^*, \bfB^*$ and the system 
is noise-free.
Let $\ell = 1$, so the first state variable is the process variable with
safety threshold $s_{1}=1$. Since $\bfA^*$ is in Jordan normal
form, $\rho(\bfA^*)=1/2<1$. When action $\bfu_0$ is applied we obtain
$\bfx_1 = \begin{bmatrix}1/4& 1\end{bmatrix}'$, which is safe. However
for any $\bfu_1=\begin{bmatrix}u_1&u_2\end{bmatrix}'$,
  \[
  \bfx_1 = \bfA^* \bfx_1 + \bfB^* \bfu_1
  = \begin{bmatrix}9/8\\1+u_1\end{bmatrix}
  \]
which is unsafe---this means the first action, while safe, ensures
that no next action is safe. Thus, unless our safety goals are fully
myopic, one-step safety will not suffice.
\end{example}

\section{Trajectory Safety}

Despite having an easily checkable criteria, one-step safety is 
insufficient to guarantee safety over a longer duration. In this section,
we focus on algorithms for verifying trajectory safety for a given sequence of
actions, including running a fixed action over an extended duration.
In particular, we develop an algorithm for finding a maximum ball of
safe actions centered around a known safe action. This allows one to
run, for example, a sequence of randomly chosen actions within the safe
action ball. First, we start off with some basic facts for linear
Gaussian trajectory rollouts.

\subsection{Linear Gaussian Rollouts}
Assuming linear Gaussian transitions, we can rollout the state at
time $\tau\ge 1$ given action sequence $\bfu_{0:\tau-1}$
\begin{align}
\bfx_\tau &= \bfA^\tau \bfx_0 + \sum_{t=0}^{\tau-1} \bfA^t \bfB \bfu_t +
\bfA^t \bfxi_t \\
\Exp[\bfx_{\tau}] &= \bfA^\tau \bfx_0 +
\sum_{t=0}^{\tau-1} \bfA^t \bfB \bfu_t
\end{align}
  As
$\tau\to\infty$, $\Exp[\bfx_\tau]$ might diverge. To ensure
convergence, we need the following \emph{Schur stability} \cite{laselle:1986}
assumption.
\begin{assumption}
Let $\theta = (\bfA, \bfB)$ be the linear Gaussian parameters of
a transition model. We assume the spectral radius of $\bfA$
is less than $1$. That is, all eigenvalues lie in the open unit circle
\[
\rho(\bfA) = \max \{ |\lambda| ~:~ \det(\bfA - \lambda \bfI_d) = 0 \} < 1.
\]
\end{assumption}
$\bfA^\tau$ can grow arbitrarily large if $\rho(\bfA) > 1$, and for
$\rho(\bfA)=1$ it may diverge or converge depending on properties of
$\bfA$. Schur
stability is a reasonable assumption in many
applications since we may be aware of an action that when applied at
each time step eventually converges to a safe steady state. As
$\tau\to\infty$ the effect of the initial state $\bfx_0$ diminishes since
$\bfA^\tau \bfx_0 \to {\bf 0}$.

The variance of the state at time $\tau$ is independent of $\bfu_{0:\tau-1}$ and $\bfx_0$ and can
be derived as
\begin{equation}
\Var[x_{\tau \ell}] = \sigma^2 \sum_{t=0}^{\tau-1} \| \bfA^t[\ell,:] \|^2.
\label{eq:var}
\end{equation}
Let $\bfR^{(\ell)} \in \R^{d\times d}$ with $\bfR^{(\ell)}_{\ell \ell} = 1$ and $0$
else where. We can rewrite the variance as
\begin{align*}
  \Var[x_{\tau \ell}] &= \sigma^2 \left[\sum_{t=0}^{\tau-1} \bfR^{(\ell)} \bfA^t (\bfA^t)' \bfR^{(\ell)}\right]_{\ell \ell} 
\end{align*}
so we have,
\begin{align*}
  \Var[\bfx_\tau ] &= \sigma^2 \sum_{\ell=0}^d\sum_{t=0}^{\tau-1} \bfR^{(\ell)} \bfA^t \bfA'^t \bfR^{(\ell)}\1 \\
  &= \sigma^2 \sum_{\ell=0}^d \bfR^{(\ell)} \left[\sum_{t=0}^{\tau-1}  \bfA^t \bfA'^t\right] \bfR^{(\ell)}\1 \\
  &= \sigma^2 \sum_{\ell=0}^d \bfR^{(\ell)} S(\tau) \bfR^{(\ell)}\1 
\end{align*}
where $S(\tau) = \sum_{t=0}^\tau \bfA^t \bfA'^t$ and $\1$ is the
all-one vector.  With Schur stability, the variance is
bounded:
\begin{theorem}
If $\rho(\bfA)<1$, 
$\lim_{\tau\to\infty}\Var[\bfx_{\tau}]$ exists.
\label{thm:ssvar}
\end{theorem}
\begin{proof}
To show the limit exists, let $\epsilon > 0$. 
We apply Gelfand's formula \cite{lax:2002} which states 
$\rho(\bfA) =\lim_{k\to\infty} \| \bfA^k \|^{1/k}$ for any matrix
norm $\| \cdot \|$. By Gelfand's, there exists a small enough $\epsilon' > 0$
and a large enough $K$
such that for all $k \ge K$ the following two conditions hold,
\begin{align*}
\| \bfA^k\|_F^{1/k} \le \rho(\bfA) + \epsilon' \quad\textrm{and}\quad 
\frac{(\rho(\bfA) + \epsilon')^k}{1-\rho(\bfA)-\epsilon'} < \epsilon.
\end{align*}
Then we have,
\begin{align*}
\frac{1}{\sigma^2}\Var[x_{k\ell}] &\le \sum_{t=0}^{K-1} \| \bfA^t[\ell,:]\|^2 + \sum_{t=K}^{k}\|\bfA^t\|^2_F \\
&\le \sum_{t=0}^{K-1} \| \bfA^t[\ell,:]\|^2 + \sum_{t=K}^k(\rho(\bfA)+\epsilon')^{2t}\\
&< \sum_{t=0}^{K-1} \| \bfA^t[\ell,:]\|^2 + \frac{(\rho(\bfA)+\epsilon')^{2}}{1-(\rho(\bfA)+\epsilon')^{2}} \\
&< \sum_{t=0}^{K-1} \| \bfA^t[\ell,:]\|^2 + \epsilon \; .
\end{align*}
For any other $k'\ge K$ and $k' < k$, 
\[
\frac{1}{\sigma^2}\Var[x_{k\ell}] \ge \frac{1}{\sigma^2}\Var[x_{k'\ell}] \ge \sum_{t=0}^{K-1} \| \bfA^t[\ell,:]\|^2.
\]
Then we have $\frac{1}{\sigma^2}(\Var[x_{k\ell}] - \Var[x_{k'\ell}]) < \epsilon$, this shows $\Var[x_{\tau\ell}]$ is a Cauchy sequence which always has a limit.
\end{proof}
This result shows that there exist a bound $\beta(\bfA)$ such that
for all $\tau \ge 1$, $\Var[x_{\tau\ell}] < \beta(\bfA)$ so the 
variance does not grow unbounded independently of the executed action sequence.
This is reassuring for high probability safety as
$\bfx_\tau$ converges to a steady state Gaussian distribution with finite
variance. Moreover, we can compute 
high probability upper bounds for any process variable $x_{\tau\ell}$
at any future time step $\tau$. It can be computed by noticing from
Theorem~\ref{thm:ssvar} that $x_{\tau\ell}$  is a Gaussian distribution
with variance at most $\lim_{t\to\infty} \Var[x_{t\ell}]$.
Again, we can compute the expectation and add the
appropriate standard deviation multiplier to get the desired confidence bound.
That is, for the desired $1-\gamma$ confidence, choose 
$c$ such that
\[
\Pr(\bfx_t \le \Exp[\bfx_t] + c\Var[\bfx_t]) \ge 1-\gamma
\]

\subsection{Limiting Case with a Fixed Action}

We first study trajectory safety with respect to a fixed action
$\bfu$ that is applied at each time step and ask whether safety constraints
are satisfied in the steady state distribution.
Obtaining an easily computable ``safety
with high probability'' test is complicated by the
form of the variance term in Eq.~\ref{eq:var}.
For this reason, we address this problem later in
Sec.~\ref{subsec:safe-delta}.  In this section, we begin by
deriving a simpler procedure for determining safety in expectation
for a fixed action.

Recall our earlier assumption that $\rho(\bfA) < 1$. As above, suppose
we have a data set generated by past exporation such that we can
guarantee $\| \Delta \hat{\theta} \|_F \le \veps$ with
probability at least $1-\gamma$. 
Let $\theta = \begin{bmatrix}\bfA &\bfB\end{bmatrix}$ be any
``potential true model'' such that $\| \theta - \hat{\theta}\|\le
\veps$ and $\rho(\bfA)<1$. The expected steady state under $\theta$ is
\begin{align*}
\lim_{\tau\to\infty} \bfA^\tau \bfx_0 + \sum_{t=0}^{\tau-1} \bfA^t \bfB \bfu & = 
  (\bfI-\bfA)^{-1} \bfB \bfu
\end{align*}
Note that $\bfI-\bfA$ is invertible and $\lim_{\tau\to\infty}
\bfA^\tau = \mathbf{0}$ since, by assumption, $\rho(\bfA)<1$.  Recall our robust
optimization objective for safety with respect to the $\ell$-th process
variable,
\begin{align}
\max_{\bfA,\bfB ~:~\rho(\bfA)<1} & \qquad \left[\left((\bfI-\bfA)^{-1}
  \bfB\right) \bfu\right]_\ell \label{eq:traj1} \\ 
  \textrm{subject to} &\qquad
\left\|\begin{bmatrix}\bfA &\bfB\end{bmatrix} -
\begin{bmatrix}\hat{\bfA}&\hat{\bfB}\end{bmatrix}\right\|_F \le \veps \label{eq:traj2}
\end{align}

To simplify notation, we drop subscript $\ell$ from the
objective.
We use to same trick as in
one-step safety and optimize for an approximating upper bound. The
objective Eq.~\ref{eq:traj1}
can be modified by subtracting a constant term:
\begin{align}	
\left((\bfI-\bfA)^{-1} \bfB - (\bfI-\hat{\bfA})^{-1}\hat{\bfB}\right)\bfu. \label{eq:ss_obj}
\end{align}
This new formulation Eq.~\ref{eq:ss_obj} is difficult
to optimize, but we can easily maximize an 
objective that upper bounds Eq.~\ref{eq:ss_obj}
by finding an upper bound on the Frobenius norm 
of the terms inside the parentheses.  (Assume all norms below are Frobenius.) 
\begin{align}
\begin{split}
\left\|(\bfI-\bfA)^{-1} \bfB - (\bfI-\hat{\bfA})^{-1}\hat{\bfB}\right\| ={}& \left\| (\bfI-\bfA)^{-1} \bfB - (\bfI-\hat{\bfA})^{-1}
(\bfB + (\hat{\bfB}-\bfB)) \right\|
\end{split}
\nonumber \\
\begin{split}
  \le {}& \left\| (\bfI-\bfA)^{-1} - (\bfI-\hat{\bfA})^{-1}  \right\|\cdot \left\| \bfB \right\| \\
  & +\left\|(\bfI-\hat{\bfA})^{-1} \right\| \cdot \left\| \bfB-\hat{\bfB} \right\|
\end{split}
  \nonumber \\
\begin{split}
  \le {}& \left\| (\bfI-\bfA)^{-1} - (\bfI-\hat{\bfA})^{-1} \right\| \cdot \left[\left\| \hat{\bfB} \right\| +\veps\right] \\
  & +\veps \left\|(\bfI-\hat{\bfA})^{-1} \right\| \label{eq:ss_obj2}
\end{split}
\end{align}

Since $\| \bfA - \hat{\bfA} \| \le \veps$, let $\bfA = \hat{\bfA} -
\bfE$ for $\| \bfE \| \le \veps$. We use the following result for any
non-singular matrix $\bfM$ and small perturbation matrix $\bfE$
\cite{demmel:simat1992}:

\begin{align*}
\frac{\| (\bfM +\bfE)^{-1} - \bfM^{-1} \|}{\|\bfM^{-1}\|} \le \kappa(\bfM) 
  \frac{\|\bfE\|}{\|\bfM\|}
\end{align*}
where $\kappa(\bfM) = \|\bfM \| \cdot \|\bfM^{-1}\|$ is the condition number
of $\bfM$.
By letting $\bfM = \bfI - \hat{\bfA}$, we obtain:
\begin{align}
\| (\bfI-\bfA)^{-1} - (\bfI-\hat{\bfA})^{-1} \| \le 
  \frac{\| (\bfI - \hat{\bfA})^{-1} \|^2 ~\veps}{\|\bfI -\hat{\bfA} \|}.
\label{eq:cond_ineq}
\end{align}
Substituting Ineq.~\ref{eq:cond_ineq} to Ineq.~\ref{eq:ss_obj2} gives:
\begin{multline}
\| (\bfI-\bfA)^{-1} \bfB - (\bfI-\hat{\bfA})^{-1}\hat{\bfB} \| \le \\
  \veps \| (\bfI - \hat{\bfA})^{-1} \| \left[
  \frac{\| (\bfI - \hat{\bfA})^{-1} \|}{\|\bfI -\hat{\bfA} \|} \cdot (\| \hat{\bfB} \| +\veps) +1
  \right] . \label{eq:cond_ineq2}
\end{multline}
Thus we obtain an upper bound the original optimization
Eqs.~(\ref{eq:traj1}-\ref{eq:traj2}) as follows:
\begin{align}
\max_{\bfc} & \qquad \bfc' \bfu + (\bfI -\hat{\bfA})^{-1}\hat{\bfB} \bfu \\
\textrm{s.t.} & \qquad \|\bfc\|\le\textrm{RHS of (\ref{eq:cond_ineq2})} \; .
\end{align}
Using the same trick as in the one-step case, we can solve for the optimal
value by finding the largest singular value. This gives the following upper
bound on the optimal objective value of the original maximization problem:
\begin{multline}
  V^*(\hat{\bfA},\hat{\bfB}, \veps, \bfu) = 
  \veps \| (\bfI - \hat{\bfA})^{-1} \| \left[
  \frac{\| (\bfI - \hat{\bfA})^{-1} \|}{\|\bfI -\hat{\bfA} \|} \cdot (\| \hat{\bfB} \| +\veps) +1 \right]
  \\ \cdot \left( \|\bfu\| +  [(\bfI -\hat{\bfA})^{-1}\hat{\bfB}]_{\ell}\right).
\end{multline}
To check the safety in expectation of the $\ell$-th process variable, in the
limit, when we execute action $\bfu$ at every time step, we simply test
whether $V^*(\hat{\bfA},\hat{\bfB}, \veps, \bfu) > \spt_\ell$.
Note that this upper bound test will generally be overly conservative---
as a consequence, we propose heuristics later in the section that should
be tighter in practice.

\subsection{Safety within $\delta$-Ball}
\label{subsec:safe-delta}

Assume, as discussed above,
we have determined that $\bfu$ is a safe action for the next $T$ time steps
under any transition model $\theta$ such that $\| \theta -
\hat{\theta} \|_F\le\veps$ with some probability (again, it's assumed
that $\Delta \| \hat{\theta} \|_F \le \veps$). For purposes
of system identification, we need to explore and not just execute
a fixed action. In this section, we discuss how to
determine whether a \emph{given region of action space} is safe (either
in expectation or with high probability). We do this in a way that exploits
the notion of a fixed safe action discussed above.

Here we focus on
regions of action space defined by a $\delta$-ball of actions,
$\delta>0$ centered at a fixed safe action $\bfu$:
\[
B_\delta(\bfu) = \{ \bfz ~:~ \| \bfz -\bfu\|_2 \le \delta \}.
\]
To verify that
region $B_\delta(\bfu)$ is safe
(see Def.~\ref{def:safe-region}),
we must ensure that any sequence of
actions $(\bfu+\bfz)_{0:\tau-1}$ is safe, where $\| \bfz_t - \bfu\|_2 \le
\veps$. As above, we express this problem as an
adversarial optimization for state variable $\ell$
at time $\tau$:
\begin{align}
  \max_{\bfA,\bfB,\bfz_{0:\tau-1}} & \qquad \left[\bfA^\tau\bfx_0 +
    \sum_{t=0}^{\tau-1} \bfA^{t}\bfB \bfz_t\right]_\ell +c\sqrt{\Var[x_{\tau\ell}]} \label{eq:ball_obj} \\
\textrm{subject to} & \qquad \| \bfz_t - \bfu \|\le\delta \quad \forall t \in [\tau] \label{cons:ball1} \\
& \qquad \| \begin{bmatrix}\bfA &\bfB\end{bmatrix} -
  \begin{bmatrix}\hat{\bfA}&\hat{\bfB}\end{bmatrix}
  \|_F \le \veps \label{cons:ball2} \\
& \qquad \rho(\bfA) < 1 \label{cons:ball3}
\end{align}

We address safety in expectation, i.e., we try to approximately solve
the above optimization problem when $c=0$.
Since $\bfu$ is safe, we can upper bound (\ref{eq:ball_obj}) by:
\begin{align}
V^*(\bfA,\bfB,\bfu) + \max_{\bfA,\bfB,\bfz_t} 
   \bfA^\tau \bfx_0 + \sum_{t=1}^\tau \bfA^{t-1}\bfB \bfz_t 
  &\le \\
  V^*(\bfA,\bfB,\bfu) + \max_{\bfA} \bfA^\tau\bfx_0 + \sum_{t=1}^\tau
    \max_{\bfA,\bfB,\bfz_t} \bfA^{t-1}\bfB\bfz_t \label{ineq:fixaction0}
\end{align}
For the max inside the summation, we can again try to bound 
\begin{align}
\| \bfA^{t-1} \bfB - \hat{\bfA}^{t-1} \hat{\bfB} \| \le \| \bfA^{t-1} - \hat{\bfA}^{t-1} \| (\|\hat{\bfB}\|+\veps)
+ \veps \| \hat{\bfA}^{t-1} \|  \label{ineq:fixaction1}
\end{align}
Let $\bfA = \hat{\bfA} - \bfE$ for $\| \bfE \| \le \veps$, we have
\begin{align}
  \| \bfA^{t-1} - \hat{\bfA}^{t-1} \| &= \| \left(\hat{\bfA} - \bfE\right)^{t-1} - \hat{\bfA}^{t-1} \| \nonumber \\
  &= \left\|\left( \sum_{\substack{j_1,\ldots,j_{2t-2}\ge 0 \\j_1 + \cdots + j_{2t-2} = t-1}} \prod_{i=1}^{t-1} \hat{\bfA}^{j_{2i-1}} \bfE^{j_{2i}} \right) -  \hat{\bfA}^{t-1} \right\| \nonumber \\
  &= \left\|\sum_{\substack{j_1,\ldots,j_{2t-2}\ge 0 \\j_1 + \cdots + j_{2t-2} = t-1\\j_2 + j_4 + \ldots + j_{2t-2} > 0}} \prod_{i=1}^{t-1} \hat{\bfA}^{j_{2i-1}} \bfE^{j_{2i}} \right\| \nonumber \\
  &\le \sum_{i=1}^{t-1} \binom{t-1}{i} \|\hat{\bfA} \|^i \veps^{t-i-1} \label{ineq:matrixbinom} \\
  &= \left(\|\hat{\bfA}\| + \veps \right)^{t-1} - \| \hat{\bfA} \|^{t-1} \nonumber
\end{align}
where Ineq.~\ref{ineq:matrixbinom} follows from triangle inequality
and sub-multiplicative property of Frobenius norm. Substituting this
inequality into Ineq.~\ref{ineq:fixaction1}, and then into
Ineq.~\ref{ineq:fixaction0}, gives an upper bound on the objective
(\ref{eq:ball_obj}) as a function of the $\bfz_t$'s. This can
be solved
again by computing the maximum singular value of the resulting matrix.

However, this bound for testing safety in expectation will generally
be quite loose, yielding a conservative test of the safety
of $B_\delta(\bfu)$. To compensate, we instead consider
a direct approximation of the above optimization.

\subsection{Computing Largest Safe $\delta$-ball}
\label{subsec:compute-largest}

Due to the looseness of the closed-form upper bound above, we can instead
computationally optimize adversarial choices of $\bfA, \bfB, \bfb$ and
$\bfz_t$ that attempt to violate one of the process variables.
We optimize (\ref{eq:ball_obj}-\ref{cons:ball3}) using alternating
optimizing over $\bfA$, $\bfB$ and $\bfz_t$. We start with a
random initialization (e.g., uniform random) of these
parameters. We first optimize $\bfA$ given
$\bfB$ and $\bfz_t$. Note that by Eq.~(\ref{eq:ball_obj}), the value of
each process variable is a simple multi-variate polynomial over the
parameters of $\bfA$, including the additive standard deviation
multiple giving us the desired Gaussian confidence bound. We can
optimize this using gradient ascent (and the resulting
local maxima can be improved using random restarts).

We can then fix $\tilde{\bfA}$, and $\tilde{\bfz_1}, \ldots,
\tilde{\bfz_\tau}$.  Let $\tilde{\bfA}^t_{ij} = \tilde{a}(t)_{ij}$,
then $(\bfB \tilde{\bfz}_t)_i = \sum_{j=1}^m B_{ij}
\tilde{z}_{tj}$. The variance term remains constant, giving:
\begin{align*}
\sum_{t=0}^{\tau-1} \tilde{\bfA}^{t}[\ell, :] \bfB \tilde{\bfz}_t &= \sum_{i=1}^d \tilde{a}(t-1)_{\ell i} \sum_{j=1}^d B_{ij} \tilde{z}_{tj} \\
&= \sum_{t=1}^\tau \sum_{i=1}^d \sum_{j=1}^d B_{ij} \tilde{a}(t-1)_{\ell i} \tilde{z}_{tj} \\
&= \mathrm{vect}(\bfB) \cdot \left( \sum_{t=1}^\tau \tilde{\bfA}^{t-1}[\ell, :]  \otimes \tilde{\bfz}_t \right)
\end{align*}
where $\mathrm{vect}(\cdot)$ is the reshaping of $\bfB$ into vector
form, and $\otimes$ is the tensor product. One can again solve this by
formulating it as a maximum singular value problem.

Finally, suppose $\tilde{\bfA}$ and $\tilde{\bfB}$ are fixed. Each
$\bfz_t$ can be optimized independently in the summation
$\tilde{\bfA}^t[\ell, :] \tilde{\bfB} \bfz_t$ subject to $\| \bfz_t -
\bfu \|_2 \le \delta$. This again can be posed as a maximum singular
value problem.
Thus, for a fixed $\delta$ we can approximate a lower bound on the
optimization, which tells us whether there exists an adversarial
example that makes $\delta$ unsafe. See
Algorithm~\ref{alg:safeball}. Since $\delta$ is a scalar, we can do
a simple ``galloping'' search to find the largest safe
$\delta$. See Algorithm~\ref{alg:maxsafeball}

\begin{algorithm}[th]
\caption{$\safeball$ Approximate Safe Ball}
\label{alg:safeball}
 \begin{algorithmic}[1]
   \REQUIRE $\bfu$, $\delta$, ($\hat{\bfA}, \hat{\bfB}, \veps)$, trajectory length $T$
 \FOR{$\tau=1..T$}
 \STATE Randomly initialize $\bfA$, $\bfB$, $\bfz_{0:\tau-1}$
 \REPEAT
 \STATE $(U, \boldsymbol\Sigma, \bfV) = \svd\left( \sum_{t=1}^\tau \bfA^{t-1}[\ell, :]  \otimes \bfz_t \right)$
 \STATE $\bfB \leftarrow \veps U \bfV[0,:] + \hat{\bfB}$
 \FOR{$t=1..\tau$}
 \STATE $(U', \boldsymbol\Sigma', \bfV') = \svd\left(\bfA^{t-1}[\ell,:]\bfB\right)$
 \STATE $\bfz_{t-1} \leftarrow \delta U' \bfV'[0,:] + \bfu$
 \ENDFOR
 \STATE Update $\bfA$ via gradient ascent on Eq. (\ref{eq:ball_obj}) subject to Ineqs. (\ref{cons:ball2}, \ref{cons:ball3})
 \UNTIL{user defined termination condition}
 \IF{value of Eq. (\ref{eq:ball_obj}) $>\spt_{\ell}$}
 \RETURN FALSE
 \ENDIF
 \ENDFOR
 \RETURN TRUE
\end{algorithmic}
\end{algorithm}

\begin{algorithm}[t]
\caption{$\maxsafeball$ Approximate Maximum Safe Ball}
\label{alg:maxsafeball}
 \begin{algorithmic}[1]
 \REQUIRE ($\hat{\bfA}, \hat{\bfB}, \veps)$, trajectory length $T$, nominal safe action $\sact$, initial $\delta_0$ such that $B_{\delta_0}(\sact)$ is safe, tolerance $\tol$.
 \STATE $l \leftarrow 0, h\leftarrow\delta_0, \delta \leftarrow \delta_0$
 \WHILE{$\safeball(\sact, \delta, \hat{\bfA}, \hat{\bfB}, \veps, T)$}
 \STATE $l \leftarrow \delta, h\leftarrow 2\delta$
 \STATE $\delta\leftarrow 2\delta$
 \ENDWHILE
 \WHILE{$h-l>\tol$}
 \STATE $\delta\leftarrow (h-l)/2$
 \ENDWHILE
 \RETURN $l$
\end{algorithmic}
\end{algorithm}
The $\safeball$ algorithm also allows us to determine whether a given
sequence of actions, $\bfu_{0:T-1}$, is safe. In $\safeball$ we just
need to fix $\bfz_t = \bfu_t$ and update only $\bfA$ and $\bfB$ in
each iteration of optimization.

The use of a single ball may not be ideal to get diversity of training
data. Instead we can generate new fixed, safe actions from
which to build new safe balls using the above algorithms. We
maintain a list of safe balls from which we update its safe radius and
from which we can create new safe balls. For exploration, we may
randomly select a safe ball from the list and run exploration by, say,
randomly choosing actions within the ball. In our experiments, we
elaborate on this method in more detail and illustrate how it improves
sample efficiency with respect to model errors.


\section{Experiments}

We now describe two sets of experiments designed to test our algorithms for
safely identifying the true underlying model of system dynamics
by exploring within 
growing safe $\delta$-balls.  
In the first set of  experiments, we test how the largest
safe $\delta$-ball $B_\delta(\sact)$ increases in size as a function of
$\veps$, the model approximation error. In the second, we
compare how quickly the true model can be identified by gathering
exploration data using: (a) a single safe ball; or (b) multiple safe
balls. In both experimental settings, we test for safety in expectation.

We test out methods on the control on fancoil operation on a data 
center floor---fancoil control dictates the amount of airflow over section
of the DC floor, as well as the flow of water (used to cool returning
hot air) through the fancoils themselves. Typically,  very large number
of fancoils are distributed throughout the DC and can have significant
interactions.  In our experiments, we use a model of
fancoil and air temperature dynamics that is
a modified, normalized model of a single fancoil (including
its interactions with other nearby fancoils) derived from
a DC cooling system at Google. 
The control problem is discrete-time system with time steps of
of 30 seconds.  There modified model has five state variables ($d=5$)
actions comprised of two control
inputs ($d'=2$) corresponding to fan speed and water flow. The ``true''
model to be identified in our experiments has a
spectral radius $\rho(\bfA^*)=0.90$. A single state variable (the first)
is the only process variable under control. The nominal
safe action $\sact = (1.833, 1.857)$ is known, based on expert knowledge
of the domain.
We set the trajectory length to be $T=20$, which is a sufficient horizon
for, say, model-predictive control in this model.
We set tolerance $\tol = 0.05$
for $\safeball$'s galloping search.

In the first experiment, we vary the model error parameter $\veps$
and ran the $\maxsafeball$ algorithm to find the largest safe
$\delta$-ball. We set the input parameter estimates $\hat{\theta}
= \begin{bmatrix}\hat{\bfA} &\hat{\bfB}\end{bmatrix}
= \begin{bmatrix}\bfA^* &\bfB^*\end{bmatrix} = \theta^*$. That is, the
  algorithm will optimize for the worst $\theta\in B_{\veps}(\theta^*)$
  that (potentially) violates the safety constraints.
  We use an initial safe
  radius of $\delta_0 = 0.1$ around the safe action $\sact$ at which to
  begin the $\safeball$ galloping search.
 %
%
We varied $\veps$
from $0$ to $0.0425$ with more granularity near $\veps = 0$. For some
context on the model estimation error $\veps$, we note that
$\| \theta^* \|_F = 1.974$; hence the largest relative error for $\veps = 0.0425$ is $\veps / \| \theta^* \|_F \approx 2.15\%$. Any relative error greater
than $2.15\%$ results in a zero radius safe ball---containing only the
nominal safe action. We vary the relative error from $0\%$ up to $2.15\%$.
At $\veps
= 0$, when the true parameters are known, we get the largest safe 
$\delta$-ball. 
Results are shown in Fig.~\ref{fig:eps_vs_delta}.
The plot shows a non-linear relationship where
the largest computed $\delta$ drops at a faster rate for smaller
$\veps$ compared to larger values of $\veps$. The largest $\delta$
drops to $0.2375$ when we reach $\veps = 0.0425$.
%

For the second experiment we simulated a typical use case of gradually
identifying the model parameters while running safe exploration to
generate training data. In particular, we compare the sample
efficiency of running exploration on one dynamically growing safe ball
versus many, simultaneously, dynamically growing safe balls.
In the true model we add Gaussian noise with standard deviation
$\sigma = 0.01$. We set a fixed constraint on the process variable for
both the single ball and the multiple ball runs. We run $100$ episodes
each with trajectory length of $T=20$ before updating $\hat{\theta}$
and the safe balls (i.e., we update every $2000$ training examples).
Each episode starts with the same initial state. In practice, this
can be accomplished by running the nominal safe action in between
episodes until steady state is reached (driving a system to a
fixed state is a common assumption, see e.g. \cite{am:icml2012}).

When generating multiple safe balls after every $100$ episodes, we
randomly choose $2$ actions $\bfu_1, \bfu_2$ on the surface of a
randomly chosen safe ball (from an existing set of safe balls). We use
$\maxsafeball$ to find the largest $\delta_1, \delta_2$, respectively,
around those chosen actions, with $\delta_0 = 0.1$ and $\veps =
\veps_{\textrm{t}}$ the theoretical model error as given by the error
approximation bound of Eq.~\ref{eq:est_bound} with confidence fixed at
$1-\alpha=0.95$. If the largest $\delta$ computed exceeds a minimum
lower bound (we choose $0.2$) we include the new ball in our list of
safe balls. We limit the number of safe balls at $14$ to simplify our
computations and reduce clutter in our presentation. We generate
exploration actions by randomly choosing a safe ball and uniformly at
random drawing a sequence of exploration actions within that ball.

Fig.~\ref{fig:n_vs_eps} contrasts the reduction in theoretical model error
$\veps_{\textrm{t}}$ versus actual model error $\veps_{\textrm{a}} = \|
\hat{\theta} - \theta^* \|$ under the two exploration
regimes. It shows that exploration that generates actions from multiple
safe balls is generally more sample efficient (i.e., it achieves the same
model error with fewer examples). For instance, after $1000$ episodes
($n=20000$ examples) multi-ball exploration achieves
$\veps_{\textrm{t}} = 0.0224$ and $\veps_{\textrm{a}} = 0.0045$ while
exploration with a single ball achieves $\veps_{\textrm{t}} = 0.0305$
and $\veps_{\textrm{a}} = 0.0075$, which is considerably worse 
As $n$ increases, both theoretical and actual model error decrease at
a slower rate (as expected). Moreover, the gap in sample efficiency
between the two regimes gets smaller at $n$ grows.

We note that
theoretical model error is conservative estimate of the
actual model error. Furthermore, actual model error is a conservative
measure when using the model to optimize for a good control policy,
which does not necessarily require a fully accurate model.  In terms
of exploration time, $n=20000$ examples corresponds to $166$ hours or
about $6$ days and $22$ hours on a single fancoil. However, ``practical''
identification will proceed much more quickly for two reasons. First,
as discussed above, a data center has many fancoils, this
exploration time can be reduced
significantly
by exploring across multiple fancoils simultaneously. Second,
we generally do \emph{not} require accurate system identification
across the full range of dynamics: for example,
once model confidence is sufficient for specific regions of state-control
space such that these points can be shown to never (or ``rarely'') be
reached under optimal control, further model refinement in unneccesary (i.e.,
the value of information for more precise identification is low).

Fig.~\ref{fig:many_balls_cmp} shows the safe ball(s) after being
updated for various values of $n$ (i.e. immediately after $n$ examples
have been generated) for the two regimes. One can see that the area of
action space covered when exploring with many safe balls grows more
quickly than with just one ball. In fact, the original safe ball
around $\sact$ is larger in the many balls regime due to the smaller
$\veps_{\textrm{t}}$ that results from better sample efficiency. For
values of $n \ge 50000$ the balls in both regimes do not grow as
quickly, this corresponds to the model errors starting to decline at a
slower rate.

\begin{figure}
\centering
\includegraphics[scale=.55]{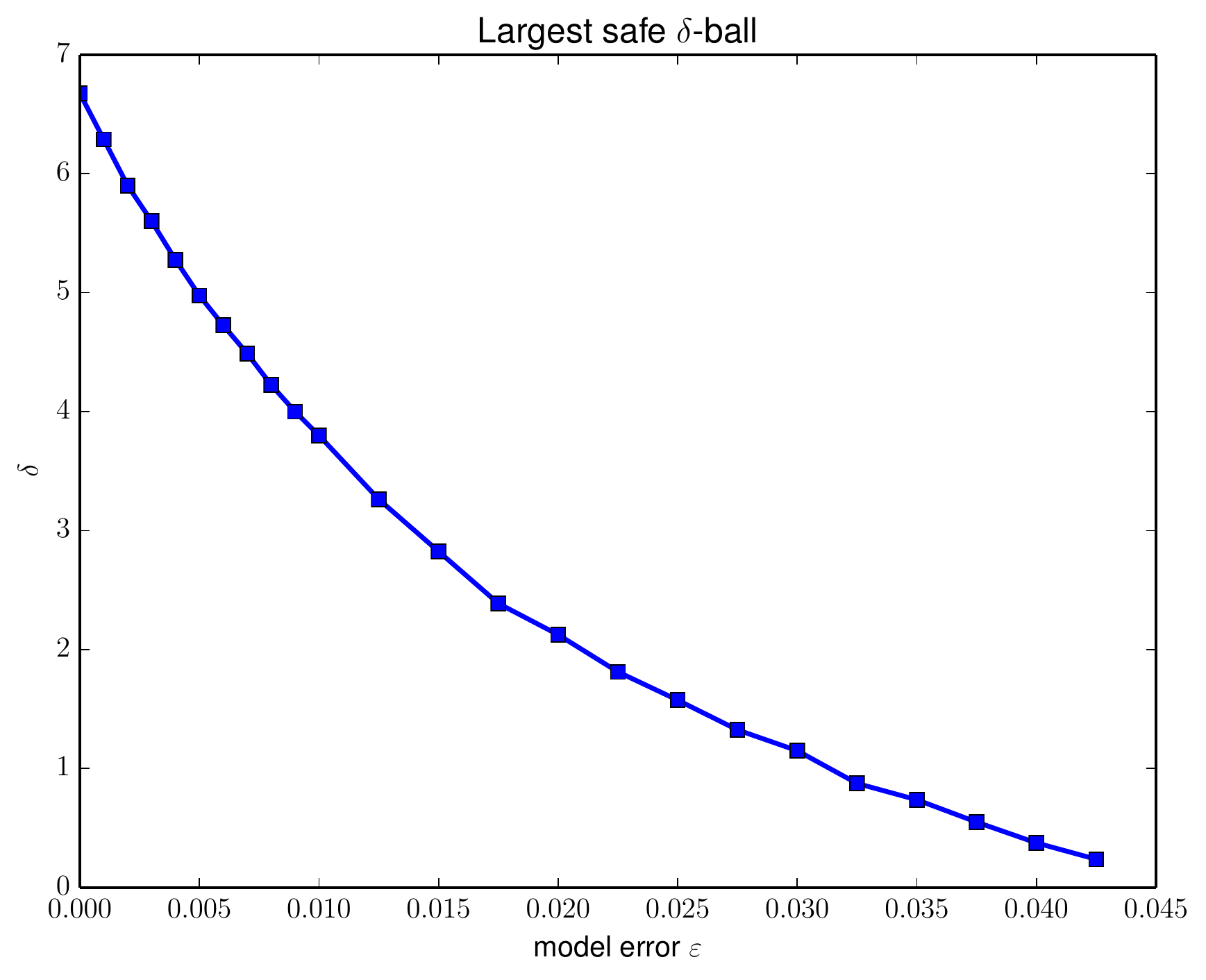}
\footnotesize{\caption{The largest safe ball around the nominal safe action $\sact$ as approximated by $\maxsafeball$ for various input model errors $\veps$.}
\label{fig:eps_vs_delta}}
\end{figure}

\begin{figure}
\centering \includegraphics[scale=.4]{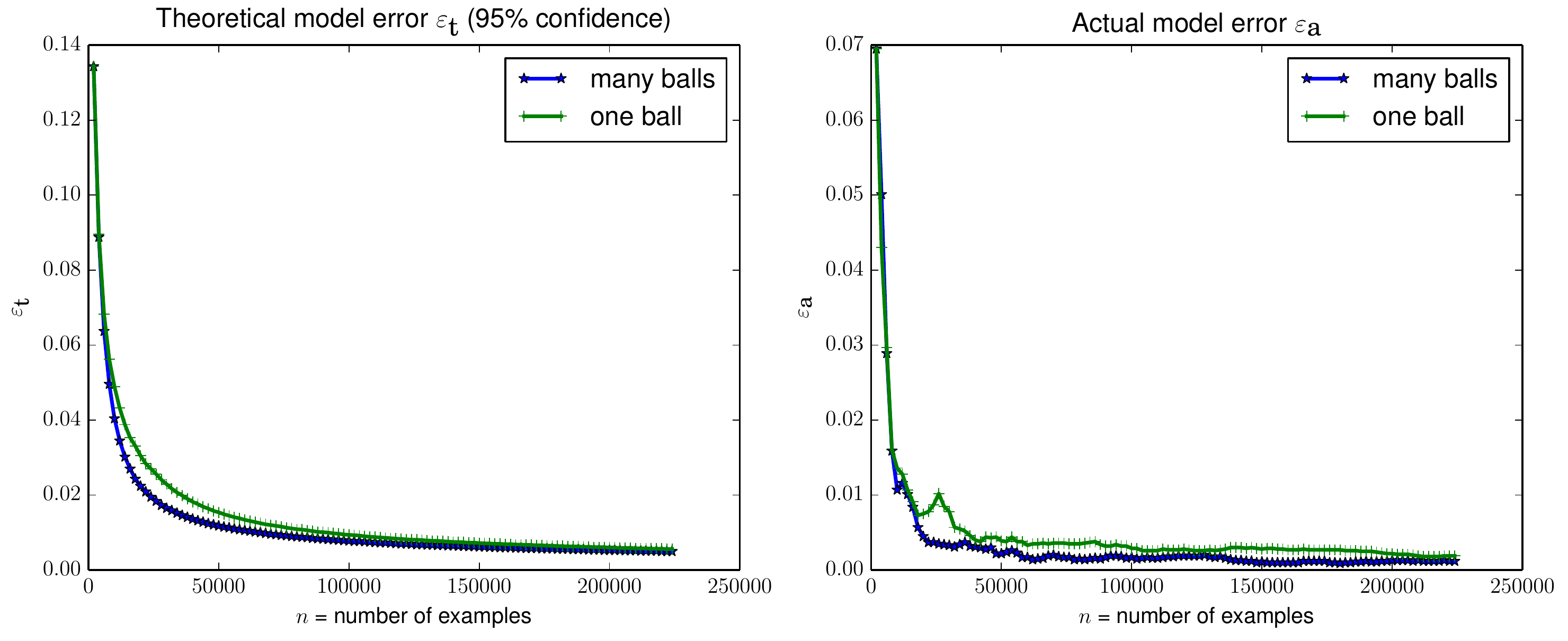}
\footnotesize{\caption{Comparing the theoretical and actual model
    errors under two safe exploration regimes: (1) random actions
    chosen from up to 14 safe balls or (2) random actions chosen from
    a single ball.}
    \label{fig:n_vs_eps}}
\end{figure}

\begin{figure}
\centering
\includegraphics[scale=.4]{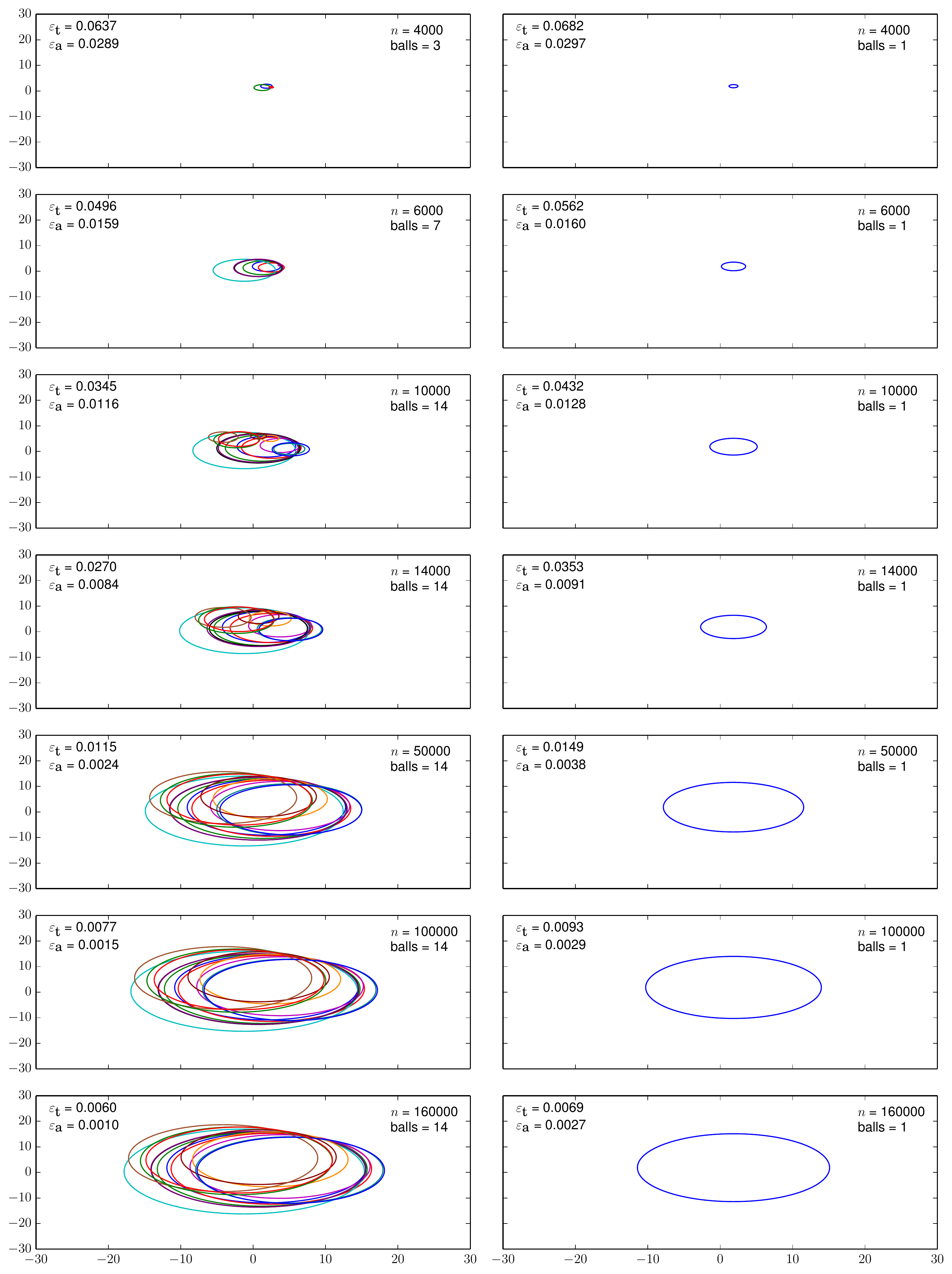}
\footnotesize{\caption{
    Left column plots show the gradual addition of
safe balls (i.e., circles, up to $14$) and growth of each ball. Each plot has $x$ and $y$ axis
corresponding to the two control inputs. Each circle within a plot 
indicates a safe action region where any sequence of $T$ actions 
is safe. Plots on the right shows gradual growth of a single safe ball.}
    \label{fig:many_balls_cmp}}
\end{figure}



\section{Concluding Remarks}


We studied safe exploration for identifying system parameters of a
linear Gaussian model. In our approach we defined safety as satisfying
linear inequality constraints on the state at each time step of a
trajectory. This is motivated by applications involving control of
process variables and where violating those constraints may lead to
significant negative impacts. We showed how to compute many regions of
safe actions when given a nominal safe action. Confined to taking
actions within each safe region, any exploration strategy may be
applied and theoretical bounds on the approximation quality of an
estimated model can be computed online. By computing and exploring within
many safe regions, our experiments show we can increase the
sample-efficiency of exploration. For future work, we would like to
extend safe exploration to not just identify the model as accurately
as possible, but to gather just the right kind of transition samples
that optimizes the quality of the controller.

\subsubsection*{Acknowledgments}
We like to thank Nevena Lazic, Eehern Wang and Greg Imwalle for useful discussions.

\bibliographystyle{plain}
\bibliography{long,standard}

\end{document}